\documentclass[letterpaper, 10pt, conference]{ieeeconf}

\let\labelindent\relax 

\makeatletter
\newcommand{\disablepackage}[2]{%
  \disable@package@load{#1}{#2}%
}
\newcommand{\reenablepackage}[1]{%
  \reenable@package@load{#1}%
}
\makeatother

\usepackage{times}

\usepackage{cite}
\usepackage{multicol}
\usepackage[bookmarks=true]{hyperref}

\usepackage{amsmath, amsfonts, amssymb, amsthm}
\usepackage{enumerate}
\usepackage[inline]{enumitem}
\usepackage{mathtools}
\usepackage{graphicx}
\usepackage{longtable,tabularx}
\usepackage{placeins} 
\usepackage{float}
\usepackage{multirow}
\usepackage{bbm}
\usepackage[table]{xcolor}
\usepackage{threeparttable}
\usepackage{balance}
\usepackage[ruled,algo2e]{algorithm2e}
\usepackage{algpseudocode}
\usepackage{adjustbox}
\usepackage{booktabs}
\usepackage{bm}
\usepackage{etoolbox}
\usepackage{microtype}
\usepackage{svg}

\usepackage[capitalise]{cleveref}
\usepackage{units}
\usepackage{cancel}

\newtheorem{theorem}{Theorem}

\newtheorem{proposition}{Proposition}

\newtheorem{corollary}{Corollary}

\newcommand{\R}{\mathbb{R}}
\newcommand{\mb}[1]{\mathbf{#1}}

\newcommand{\diag}{\mathrm{\mathbf{diag}}}

\newcommand{\minimize}[1]{ \ensuremath{\underset{#1}{\mathrm{minimize}}\ }}

\newcommand*{\subj}{\ensuremath{\mathrm{subject\ to\ }}}

\newcommand{\mbb}[1]{\mathbb{#1}}

\newcommand{\bzero}{\mathbf{0}}

\newbool{condensed_paper}
\setbool{condensed_paper}{true}

\newbool{further_condense}
\setbool{further_condense}{true}

\newbool{show_slower_image}
\setbool{show_slower_image}{false}

\makeatletter
\newcommand{\oldnormaux}[3]{\mathpalette\oldnormaux@i{{#1}{#2}{#3}}}
\newcommand{\oldnormaux@i}[2]{\oldnormaux@ii#1#2}
\newcommand{\oldnormaux@ii}[4]{%
  \sbox\z@{$\m@th#1#2#4#3$}%
  \sbox\tw@{$\m@th\|$}%
  \mathopen{\hbox to\wd\tw@{\hss\vrule height \ht\z@ depth \dp\z@ width .2\wd\tw@\hss}}%
  #4
  \mathclose{\hbox to\wd\tw@{\hss\vrule height \ht\z@ depth \dp\z@ width .2\wd\tw@\hss}}%
}
\makeatother

\makeatletter
\newcommand{\longdash}[1][2em]{%
  \makebox[#1]{$\m@th\smash-\mkern-7mu\cleaders\hbox{$\mkern-2mu\smash-\mkern-2mu$}\hfill\mkern-7mu\smash-$}}
\makeatother
\newcommand{\omitskip}{\kern-\arraycolsep}

\newcommand{\p}{\mathbf{p}}

\newcommand{\vel}{\mathbf{v}}
\newcommand{\accel}{\mathbf{a}}

\newcommand{\yhat}{\hat{y}}

\newcommand{\yopthat}{\hat{y}^*}

\IEEEoverridecommandlockouts

\IEEEoverridecommandlockouts

\title{A Control Barrier Function for Safe Navigation with Online Gaussian Splatting Maps}
\author{Timothy Chen$^{1*}$, Aiden Swann$^{1*}$, Javier Yu$^{1*}$, Ola Shorinwa$^{1}$, Riku Murai$^{2}$,\\ Monroe Kennedy III$^{1}$, Mac Schwager$^{1}$
\thanks{$^{*}$ The co-first authors contributed equally.}%
\thanks{$^{1}$ Stanford University,~Stanford,~CA, USA.}%
\thanks{$^{2}$ Imperial College London, London, UK.}%
\thanks{Corresponding author: \tt\small {chengine@stanford.edu}.} %
\thanks{T. Chen was supported by a NASA NSTGRO fellowship and A. Swann was supported by NSF GRFP Fellowship No. DGE-2146755. This work was also supported in part by ONR grant N00014-23-1-2354, NSF FRR grant 2342246, and MIT Lincoln Labs grant 7000603941. Toyota Research Institute provided funds to support this work.}
}

\begin{document}

\maketitle
\begin{abstract}
SAFER-Splat (\underline{S}imultaneous \underline{A}ction \underline{F}iltering and \underline{E}nvironment \underline{R}econstruction) is a real-time, scalable, and minimally invasive safety filter, based on control barrier functions, for safe robotic navigation in a detailed map constructed at runtime using Gaussian Splatting (GSplat). We propose a novel Control Barrier Function (CBF) that not only induces safety with respect to all Gaussian primitives in the scene, but when synthesized into a controller, is capable of processing hundreds of thousands of Gaussians while maintaining a minimal memory footprint and operating at 15 Hz during online Splat training. Of the total compute time, a small fraction of it consumes GPU resources, enabling uninterrupted training. The safety layer is minimally invasive, correcting robot actions only when they are unsafe. To showcase the safety filter, we also introduce SplatBridge, an open-source software package built with ROS for real-time GSplat mapping for robots. We demonstrate the safety and robustness of our pipeline first in simulation, where our method is 20-50x faster, safer, and less conservative than competing methods based on neural radiance fields. Further, we demonstrate simultaneous GSplat mapping and safety filtering on a drone hardware platform using only on-board perception. We verify that under teleoperation a human pilot cannot invoke a collision. 
Our videos and codebase can be found  
at \url{https://chengine.github.io/safer-splat}.

\end{abstract}

\IEEEpeerreviewmaketitle



\section{Introduction}
\label{sec:intro}

\begin{figure}
    \centering
    \includegraphics[trim={0cm 10cm 31cm 0.5cm}, width=\columnwidth]{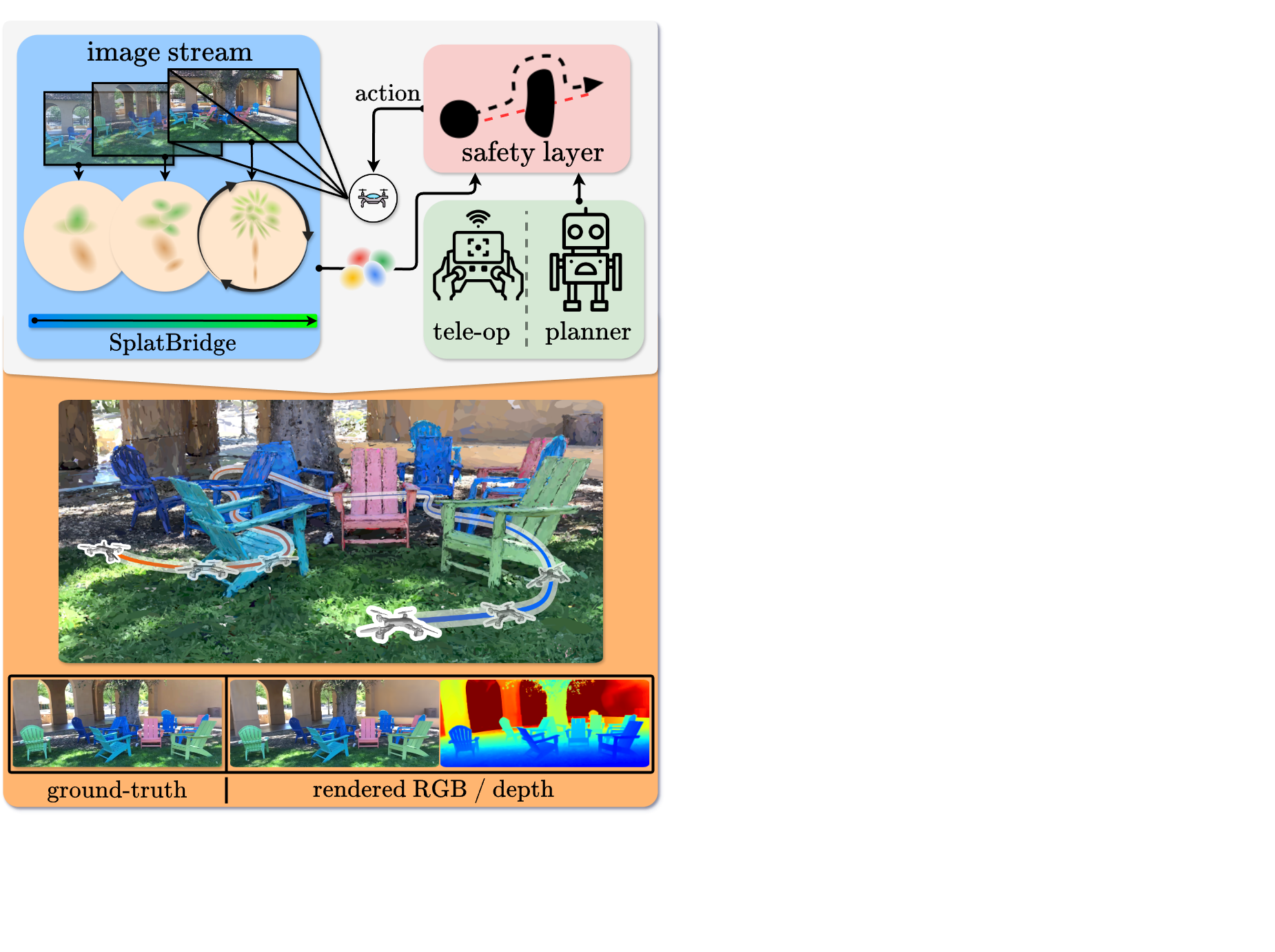}
    \caption{SAFER-Splat is a safety layer for online robotic GSplat mapping.}
    \vspace{-1em}
    \label{fig:title}
\end{figure}

Spatial understanding is an essential component of a robot's autonomy stack, and for many robotic platforms it is achieved through online mapping using the robot's on-board sensor suite. 
Mapping in robotics has traditionally been built on point-cloud and voxel-based scene representations \cite{ryde20103d, kim2018slam, palieri2020locus, jatavallabhula2023conceptfusion}. Although these representations can be designed to be lightweight, the sparsity of these representations often limits the resolution and fidelity of the resulting map. Recently, radiance fields, such as Neural Radiance Fields (NeRFs) \cite{mildenhall2020a, barron2021mip} and Gaussian Splatting (GSplat) \cite{kerbl3Dgaussians, yu2024mip, swann2024touch}, have emerged as photorealistic, high-fidelity scene reconstruction methods built off the same input modalities (i.e., RGB or RGB-D images) as their classical counterparts. In contrast to NeRFs, GSplats represent the environment using explicit, ellipsoidal primitives, lending itself to efficient, dynamic robotic manipulation and navigation algorithms \cite{lu2024manigaussian, zheng2024gaussiangrasper, shorinwa2024splat, abou2024physically, chen2024splatnav}, while offering faster training and rendering speeds than NeRFs. In this work, we propose a Control Barrier Function (CBF) safety filter for robot platforms which leverages a GSplat's simple geometric primitives for efficient computation of certifiably safe control. SAFER-Splat is a low-level safety filter, meaning it modifies higher-level commands, from a human pilot through teleop, or through a higher-level planner, to deliver collision-free actions with minimal deviation from the action intended by the higher level command.

Although algorithms for safe robotic control have been developed for a broad range of map representations, e.g., \cite{diaz2016indoor, gao2019flying, adamkiewicz2022vision, chen2023catnips, chen2024splatnav}, these existing methods either require a pre-built map, making them unsuitable for online scenarios, or require stringent assumptions on the robot's dynamics, sensing modalities, or its nominal controller. 
SAFER-Splat leverages a novel CBF, closely integrated with the GSplat representation, to derive a lightweight, minimally-invasive safety filter. To synthesize a controller, we embed the CBF safety constraint into a quadratic program, minimizing the deviation between the desired and actuated control effort. By pruning the CBF constraints to identify a minimal number of constraints, our method scales efficiently to run in real-time with hundred of thousands ellipsoidal primitives in the scene. Together, the safety guarantees, computational efficiency, and scalability of SAFER-Splat enable its superior performance in collision-free robot control in online GSplat scenes.  To showcase SAFER-Splat in a full autonomy stack, we also introduce SplatBridge, a ROS bridge for real-time, online GSplat training for robots.

In simulation, we demonstrate that SAFER-Splat outperforms other radiance field-based controllers \cite{tong2023enforcing}. We show that the proposed method is more than an order of magnitude faster, not sporadic in performance, and is consistently safer.  
In real-world experiments with a teleoperated drone, we validate that SAFER-Splat, using SplatBridge as its online GSplat module, with only data from on-board perception, successfully produces collision-free robot motion even when the operator attempts to force collisions with obstacles. While we demonstrate our CBF safety filter with SplatBridge, the concept is modular, and can be composed with other GSplat SLAM techniques (e.g., \cite{yan2023gs}) to give real-time, photo-realistic scene reconstruction and collision-free robot motion.

The contributions of SAFER-Splat can be summarized as: (1) a tractable CBF-based safety filter for safe robot navigation in GSplat maps and (2) SplatBridge, a real-time ROS bridge for training GSplats online on robot platforms.

\section{Related Work}
\label{sec:related}


\smallskip
\noindent\textbf{Radiance Fields and SLAM.} First introduced in \cite{mildenhall2020a}, Neural Radiance Fields demonstrated the photorealistic novel view rendering capability of radiance field scene representations, where the environment is represented by spatial density and color fields, parameterized by multi-layer perceptrons.
In addition to detailed scene color modeling, NeRFs produce high-fidelity, dense geometric scene reconstructions making them a potential alternative map representation for SLAM pipelines. 
Prior work \cite{sucarIMAPImplicitMapping2021, nice-slam} simultaneously optimize the NeRF network weights and the robot/camera poses. Meanwhile, NeRF-SLAM \cite{nerf-slam} and NerfBridge \cite{yu2023nerfbridge} combine existing visual odometry with online NeRF training for real-time 3D scene reconstruction.

GSplat \cite{kerbl3Dgaussians}, a successor to NeRFs, retains a similar philosophy on the parametric modeling of scene radiance fields. However, rather than using an MLP-based parametric model, GSplat represents the scene using ellipsoidal primitives with a tile-based rasterization technique for differentiable rendering. Compared to NeRFs, GSplat provides superior training speed and reconstruction accuracy. GSplat has also been integrated into SLAM pipelines \cite{yan2023gs, yugay2023gaussian, matsuki2023gaussian, keetha2024splatam}. Specifically, we integrate our safety layer with SplatBridge, a GSplat extension of NerfBridge.  



\smallskip
\noindent\textbf{Control Barrier Functions.} CBFs \cite{ames2017cbf} have become a widely utilized tool to synthesize safe controllers for robotic systems. CBFs provide guaranteed safety by ensuring forward set invariance. Typically CBFs are used in unison with a higher-level goal-oriented planner, or a human tele-operator \cite{ames2019cbf}.  In this form, CBFs act as a non-invasive \textit{safety filter} \cite{singletary2022cbf, singletary2022timecbf}. 


Some works have addressed the use of CBFs to provide safety active sensing of an environment. In \cite{sreenath2024, bolunPC}, CBFs are formulated for avoiding robot body frame point clouds from depth or LIDAR inputs. These methods do not build a map of the environment and therefore can only avoid obstacles which are currently in view of the sensors.  In \cite{Zhou_Papatheodorou_Leutenegger_Schoellig_2024, raja2024safecontrolusingoccupancy}, CBFs are created for occupancy grid based environment maps. While these methods maintain a map of the environment, they lack the scalability and expressiveness of GSplat representations.



\smallskip
\noindent\textbf{Planning and Control in Radiance Fields.} Using NeRFs as an underlying scene representation, NeRF-Nav \cite{adamkiewicz2022vision} plans trajectories for differentially flat robots, minimizing a collision cost through gradient descent, but has no safety guarantees and is slow to converge. CATNIPS \cite{chen2023catnips} converts the NeRF into a probabilistic voxel grid to plan safe paths parameterized as B\'ezier curves. Most structurally similar to our work, NeRF-CBF \cite{tong2023enforcing} uses NICE-SLAM \cite{nice-slam} to train a NeRF and uses the rendered depth map at sampled poses to enforce step-wise safety using a CBF. It requires a pre-trained NICE-SLAM map, due to the slow execution time of online NICE-SLAM. Furthermore, the discrete CBF has weak safety guarantees, and combined with slow NeRF rendering, introduces scalability issues. SAFER-Splat addresses these limitations by performing online mapping using SplatBridge and is amenable to continuous-time dynamics models.

To our knowledge, Splat-Nav \cite{chen2024splatnav} is the only work to propose a planning algorithm for a robot using a GSplat map, although we note that prior work on mapping with Gaussian Mixture Models exists \cite{goel2021rapid, goel2023probabilistic}. 
Splat-Nav converts the GSplat representation into a union of ellipsoids and performs rigorous collision checking to generate smooth, safe trajectories. Similarly, we leverage the ellipsoidal interpretation of GSplats to demonstrate safety of our safety filter; however, our safety filter is meant to intervene at the low level feedback control layer, whereas Splat-Nav provides a high-level trajectory planning layer. 


\section{3D Gaussian Splatting}
\label{sec:gsplat}


Gaussian Splatting (GSplat), recently introduced in \cite{kerbl3Dgaussians}, provides a photorealistic $3$D scene representation of an environment using $3$D ellipsoids (Gaussians) as the underlying geometric primitives. Each ellipsoid is assigned spatial and geometric attributes given by a mean ${\mu \in \mbb{R}^{3}}$ and covariance matrix ${\Sigma \in \mbb{S}_{++}}$, in addition to visual attributes given by opacity ${\alpha \in [0, 1]}$, and spherical harmonics (SH) coefficients for view-dependent visual effects. The entire representation (including all the attributes of all ellipsoids) is seeded from sparse point-cloud(s) of the environment, and optimized using a photometric error from monocular images. These point-clouds can be retrieved from structure-from-motion \cite{schoenberger2016sfm} or derived from robot odometry and sensors in online mapping. 


GSplat offers faster rendering and training times compared to NeRF-based methods, and more importantly in our use-case, provides more accurate collision geometry, as shown in \cite{chen2024splatnav}. To convert the GSplat into an interpretable representation, we leverage the fact that the minimal geometry of the Splat is the $99\%$ confidence ellipsoid for every Gaussian \cite{kerbl3Dgaussians, chen2024splatnav}. By maintaining a positive distance to this representation, a robot is guaranteed to maintain safety with respect to the geometry of the GSplat.
%
Although the resulting collision geometry from GSplat is relatively accurate, we note that a gap still exists a small gap between the ground-truth geometry and that extracted from the GSplat. We compensate for this mismatch when designing our CBF-based controller, for improved robustness.

We now formally present our definition of the scene representation which will be used by the safety layer. A scene consists of a union of $K$ ellipsoids comprising the GSplat $\mathcal{G} = \{\mathcal{E}_i\}_{i=1}^K$, where each ellipsoid $\mathcal{E}_i$ satisfies the standard form: ${\mathcal{E}_i = \{x\in \mathbb{R}^3 | (x - \mu_i)^T \Sigma_i^{-1} (x-\mu_i) \leq 1\}}$
%
for means $\mu_i$ and covariance $\Sigma_i$.
We decompose the covariances based on the rotation and scaling parameters from the GSplat formulation: $\Sigma_i = R \bar{S} \bar{S}^T R^T,$
%
where $\bar{S} = \sqrt{\chi_3^2(\gamma)} S$ represents the scales of Gaussian primitive $i$ multiplied by the $\gamma\%$th percentile of the chi-squared distribution. For this work, we use the $1\sigma$ ellipsoid with semi-major axes denoted as $S$.


\section{Control Barrier Functions}
\label{sec:cbf}

We provide a brief introduction to control barrier functions (CBFs) for double-integrator systems and refer interested readers to \cite{ames2019cbf} for a more detailed discussion of CBFs. CBFs frame safety guarantees in the lens of forward invariance. Given a CBF candidate ${h(x): \mathbb{R}^n \to \mathbb{R}}$ of the robot's state $x \in \mathcal{D}$, we define the safe set as the \mbox{${0}$-superlevel} set ${\Omega = \{ x \in \mathcal{D}\; | \; h(x) \geq 0\}}$. The $\mathrm{int}(\Omega)$ is defined where $h(x)$ is strictly positive, whereas the boundary $\partial \Omega$ is precisely where $h(x)$ is 0. In order for a safety algorithm to be safe, we must remain in $\Omega$ for all time $t\in\mathbb{R}_+$. The forward invariance can be formalized through the condition \cite{wieland2007constructive} (an extension of Nagumo's Theorem \cite{nagumo1942lage}): 
${\dot{h}(x(t)) \geq -\alpha(h(x))}$,
where $\alpha$ is any class-$\kappa$ extended function.
Intuitively, this constraint implies that the CBF candidate is allowed to decrease at an arbitrarily high rate until it reaches $\partial S$, at which point it can no longer be negative and become unsafe. Conversely, the constraint is also defined for unsafe behavior, in which the CBF is designed to increase $h(x)$ back into $\Omega$. 

In order to synthesize a controller, we 
first define the continuous-time control-affine dynamics: ${\dot{x} = f(x) + g(x) u,}$
where state transition function ${f(x): \mathcal{D}\to\mathbb{R}^n}$, actuation matrix ${g(x): \mathcal{D}\to\mathbb{R}^{n\times m}}$, and control effort $u \in \mathbb{R}^m$. 

We can define a formal CBF constraint for double-integrator systems: ${\forall x\in \mathcal{D}, \; \exists \; u:}$
\begin{equation}
\label{eq:cbf_constraint_all_degree2}
\left[ \left(\mathcal{L}_f\mathcal{L}_g h \right)^T u \geq -\mathcal{L}_f^2 h - \alpha(\mathcal{L}_f h) - \beta(\mathcal{L}_fh + \alpha(h))\right],
\end{equation}
where $\alpha, \beta > 0$ are positive selected constants. \cref{eq:cbf_constraint_all_degree2} is affine with respect to the decision variable $u$. To synthesize a controller, we embed the CBF constraint into a quadratic program of the form:
\begin{equation}
    \label{opt:cbf_qp}
    \begin{split}
    &\minimize{u} ||u - \bar{u}||_2^2\\
    &\subj \cref{eq:cbf_constraint_all_degree2},
    \end{split}
\end{equation}
which can be solved quickly to global optimality and $\bar{u}$ denotes the desired control derived from an external source.

\section{Collision Avoidance in GSplat Maps}
\label{sec:method}


\smallskip

%
\noindent \textbf{CBF Formulation.} We propose a Euclidean distance metric, the minimum distance between a spherical robot body (centered at position $\p$ with radius $r$) and an ellipsoid, in contrast to a Mahalanobis distance metric, noting that non-Euclidean distance metrics generally introduce additional computational bottlenecks. Later in this section, we extend the formulation to ellipsoid-ellipsoid inter-distances. 
Because distances are invariant to translations and rotations, we can convert the sphere-to-ellipsoid problem to a zero-mean, ellipsoid aligned frame, with the center of the ellipsoid at the origin and the primary axes oriented in the $+x, +y, +z$ directions from largest to smallest. Moreover, any reflections across the primary planes ($xy, yz, zx$) preserve distances due to symmetry of the ellipsoid. We denote variables in this translated, rotated, and reflected frame with a caret (\^{}). 

The distance between the $i$th ellipsoid and the robot centroid $\hat{\p} = F R^T(\p - \mu_i)$ can be posed as the quadratic program:
\begin{equation}
\begin{split}
\label{opt:sphere-to-ellipsoid-primal-simple}
    d_i(\hat{\p}) = & \;\minimize{\yhat} ||\hat{\p} - \yhat||_2^2\\
    &\subj \yhat^T S^{-2} \yhat \leq 1,
\end{split}
\end{equation}
where $F$ is the reflection matrix, a diagonal matrix of positive or negative 1 used to reflect points across the primary planes. The inclusion of such a matrix will be apparent when we discuss how to solve this program. We note that this program with one constraint always has a feasible solution.

\smallskip
\noindent \textbf{CBF Solution.}
Instead of solving the optimization through a solver, which is typically not scalable to millions of Gaussians on CPU and often does not provide derivatives, we choose to numerically solve the dual objective
\begin{equation}
\label{eq:lagrangian-simple}
    \max_{\lambda} \min_{\yhat} \underbrace{||\hat{\p} - \yhat||_2^2 + \lambda(\yhat^T S^{-2} \yhat - 1)}_{L(\yhat, \lambda; \hat{\p})}.
\end{equation}
Considering the case where the robot is in the safe set, the optimal variables (denoted with $*$) must satisfy first-order optimality conditions:
\begin{equation}
    \label{eq:first-order-optimality-simple}
    \begin{split}
    \frac{\partial}{\partial \lambda} L(\yhat^{\star}, \lambda^{\star}; \hat{\p}) &= \yhat^{\star T} S^{-2} \yhat^{\star} - 1 = 0 \\
    \frac{\partial}{\partial \yhat} L(\yhat^{\star}, \lambda^{\star}; \hat{\p}) &= 
    2(\lambda^{\star} S^{-2} + \mb{I})\yhat^{\star} - 2\hat{\p} = 0. \\
    \end{split}
\end{equation}
from which we can compute $\yopthat$. For scenarios in which the robot is penetrating the ellipsoid, solving these conditions will yield a non-zero penetration distance, which is crucial for a well-behaved CBF.
The optimality conditions are:
\begin{equation}
    \label{eq:first-order-optimality-direct}
\begin{split}
    \yopthat(\lambda, \hat{\p}) = & \begin{pmatrix}..., \frac{S_{ii}^2\hat{\p}_i}{\lambda + S_{ii}^2}, ... \end{pmatrix}^T\\
        q(\lambda, \hat{\p}) = &\sum_{i=1}^n \frac{y_i^{*2}}{S_{ii}^2} - 1 = \sum_{i=1}^n \frac{S_{ii}^2 \hat{\p}_i^2}{(\lambda + S_{ii}^2)^2} - 1 = 0,
    \end{split}
\end{equation}
where $\hat{\p}_i \geq \bzero$ for some reflection matrix $F$.

Note that the summation in \cref{eq:first-order-optimality-direct} is convex and monotonically decreasing (from $+\infty$ tending to 0) 
with respect to $\lambda$ for $\lambda \geq \max_i \{-S^2_{ii}\}$. Since $\hat{\p}$ is in the positive orthant, $y^*$ must also be in the positive orthant. 
Because the expression is monotonic and occupies the range $(0, \infty)$, the equation always has exactly one solution in the feasible range of $\lambda$, and can be solved to arbitrary precision using a bisection algorithm \cite{eberlyGeometricTools2013}, parallelized on GPU.

\begin{figure}
    \centering
    \includegraphics[trim={0cm 3cm 0cm 1.4cm}, width=\columnwidth]{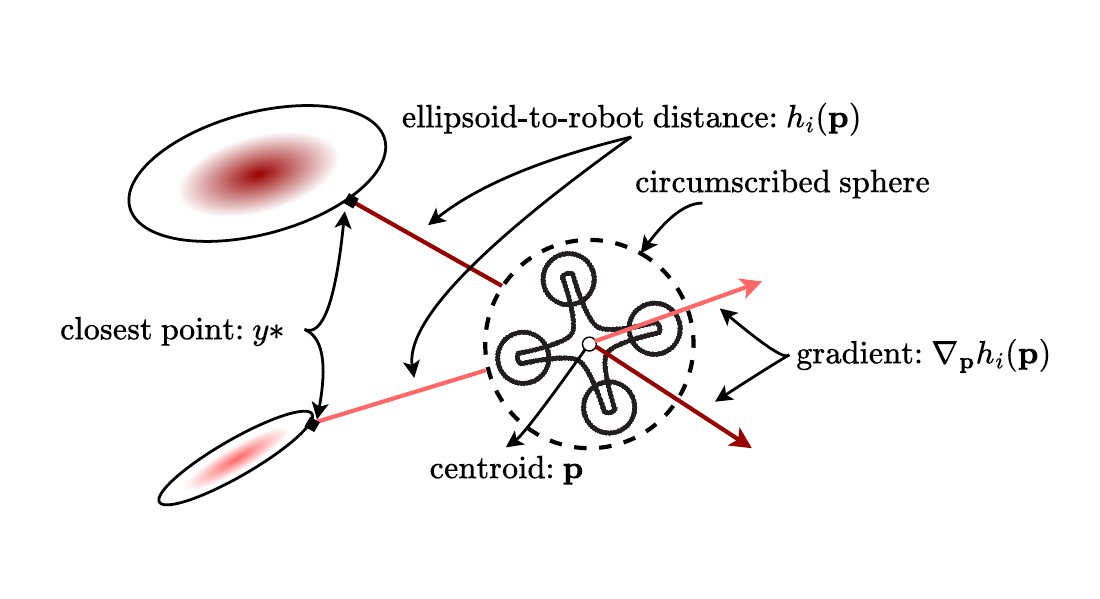}
    \caption{The CBF is posed as an optimization problem, minimizing the distance between the robot $\p$ and some point $y^*$ on the ellipsoid. We solve the robot-ellipsoid distance program through a bisection search, and return both its gradient and Hessian.}
    \label{fig:cbf-explainer}
    \vspace{-1em}
\end{figure}

    We define the robot-to-ellipsoid CBF as:
    \begin{equation}
        \label{eq:our-cbf}
        \begin{split}
        h_i(\p) &= \phi_i(\hat{\p}) d_i(\hat{\p}) - (r + \epsilon)^2\\
        \phi_i(\hat{\p}) &= \textbf{sign}(\hat{\p}^T S^{-2} \hat{\p} - 1),
    \end{split}
    \end{equation}
    which represents a tractable CBF candidate for spherical robots with radius $r$ and desired buffer distance $\epsilon$. $\gamma$ and $\epsilon$ encapsulate uncertainty in the 3D reconstruction process and in the position estimate of the robot, which should be selected based on domain-specific knowledge. Note that for $\epsilon = 0$, $h_i(\p)$ will be positive in free space, 0 when the robot body just intersects the ellipsoid, and negative when penetrating the ellipsoid.
    
    Solving \ref{eq:cbf_constraint_all_degree2} requires the derivatives of the CBF. By the Envelope theorem \cite{silberberg1999viner}, we can treat the optimized variables as constants when taking the first derivative. Hence,
    \begin{equation}
        \label{eq:cbf-grad}
        \begin{split}
        \nabla_\p h_i(\p) 
        = \phi_i \frac{\partial L}{\partial \hat{\p}} \frac{\partial \hat{\p}}{\partial \p}
        = 2\phi_i (\p - y^*)^T,
        \end{split}
    \end{equation}
    where $y^*$ is the closest ellipsoid point in the original frame.
    
    To compute the Hessian, we make use of the implicit function theorem, which states that partial derivatives of two variables related implicitly by an equation $q(\lambda, \hat{\p}) = 0$ can be expressed as: ${\frac{\partial \lambda}{\partial \hat{\p}} = -\left[\frac{\partial q}{\partial \lambda}\right]^{-1} \frac{\partial q}{\partial \hat{\p}}.}$
    %
    This theorem allows the expression of the Hessian as ${\nabla_\p^2 h_i(p) = \mathbf{H}_i}$, where:
    \begin{equation}
    \label{eq:hessian}
        \begin{split}
            \mathbf{H}_i &= 2\phi_i \left[ \mathbf{I} - \diag(\frac{S_{ii}^2}{\lambda + S_{ii}^2}) + \frac{\partial \lambda}{\partial q} \frac{\partial \yopthat}{\partial \lambda} \frac{\partial q}{\partial \hat{\p}} \right]
        \end{split}
    \end{equation}
    where ${\frac{\partial q}{\partial \lambda} = -2 \sum_{i=1}^n \frac{S_{ii}^2 \hat{\p}_i^2}{(\lambda + S_{ii}^2)^3},}$ ${\frac{\partial \yopthat}{\partial \lambda} = - \left[\frac{S_{ii}^2 \hat{\p}_i}{(\lambda + S_{ii}^2)^2} \right]_{\forall i}^T,}$ and ${\frac{\partial q}{\partial \hat{\p}} = 2\left[\frac{S_{ii}^2 \hat{\p}_i}{(\lambda + S_{ii}^2)^2} \right]_{\forall i}.}$
\cref{fig:cbf-explainer} visualizes the CBF gradients. 

\begin{proposition}
    The sphere-to-ellipsoid CBF (\ref{eq:our-cbf}) can be extended to an ellipsoid-to-ellipsoid CBF, applicable to ellipsoidal robots.
\end{proposition}
\begin{proof}
    An ellipsoidal robot in space $\mathcal{W}$ can be converted to a sphere in space $\mathcal{O}$ through a linear transformation $A: \mathcal{W} \to \mathcal{O}$ consisting of a rotation and anisotropic scaling. The set of ellipsoids are closed under linear transformations. Therefore, in $\mathcal{O}$, the sphere-to-ellipsoid CBF (\ref{eq:our-cbf}) can be used. Any action synthesized using dynamics in $\mathcal{O}$ should be mapped back to dynamics in $\mathcal{W}$. 
\end{proof}

\noindent \textbf{Controller Synthesis and Computational Efficiency.}
In this work, we utilize the double-integrator model, which is often used in CBF literature for drone dynamics \cite{sreenath2024, goncalves2024, singletary2021}:
\begin{equation}
\label{eq:double-integrator}
\begin{split}
\dot{x} = \begin{pmatrix}
\dot{\p} \\
\dot{\vel}
\end{pmatrix} = 
 \begin{pmatrix}
\vel \\
\accel
\end{pmatrix} = 
\underbrace{
\begin{pmatrix}
    \bzero \; \mathbf{I} \\
    \bzero \; \bzero
\end{pmatrix}\begin{pmatrix}
    \p \\ \vel 
\end{pmatrix}}_{f} + \underbrace{\begin{pmatrix}
    \bzero \\ \mathbf{I}
\end{pmatrix}}_{g} u,
\end{split}
\end{equation}
where $u$ is the commanded acceleration. 
This assumption is not restrictive, given that drones typically have a lower-level attitude-stabilizing controller that enables the drone to track double-integrator dynamics. Moreover, our CBF can be generalized to a higher-fidelity dynamics models for drones, as demonstrated in \cite{singletary2022cbf, cohen2024}.

%
The CBF constraint for double-integrator dynamics can be formed using \eqref{eq:cbf_constraint_all_degree2}, which can be embedded into a minimally-invasive quadratic program \eqref{opt:cbf_qp}. We prune the set of constraints in \eqref{opt:cbf_qp} for better efficiency, by reasoning about the dynamic feasibility of the system. The redundant constraints are never active, so their exclusion does not compromise safety. This type of pruning yields more savings in denser GSplats.
%



\begin{theorem}
    The quadratic program in \eqref{opt:cbf_qp} is always feasible, i.e., a feasible control input $u$ always exist for the CBF controller \eqref{opt:cbf_qp}, for robots with double-integrator dynamics, provided that ${\p \in \Omega}$. Consequently, the safe set is forward invariant, for the closed-loop system.
\end{theorem}
\begin{proof}
    Given the dynamical model \eqref{eq:double-integrator}, we compute the following Lie derivatives:
    \begin{equation*}
    \label{eq:Lie derivatives}
    \begin{split}
    \mathcal{L}_f h_i &= \left(\nabla_\p h_i \right)^T v\\
    \mathcal{L}_f^2 h_i &= f^T \left[\nabla_x^2 h_i \right] f + \left(\nabla_x h_i \right)^T \left[ \frac{\partial f}{\partial x} \right] f = v^T \mathbf{H}_i v\\
    \mathcal{L}_f\mathcal{L}_g h_i &= f^T \left[\nabla_x^2 h_i \right] g + \left(\nabla_x h_i \right)^T \left[ \frac{\partial f}{\partial x} \right] g\\
    &= \left(\nabla_\p h_i \right)^T = 2\phi_i (\p - y^*) .
    \end{split}
    \end{equation*}
    We can write \eqref{eq:cbf_constraint_all_degree2} as half-space constraints for every Gaussian:
    \begin{equation}
    \label{eq:cbf-halfspace}
        \begin{split}
    2 \phi_i (\p - y^*)^T (u + (\alpha + \beta) v) \geq & - v^T \mathbf{H}_i v - \alpha \beta h_i.
        \end{split}
    \end{equation}
    Note that for all safe states $x$ such that $h_i(x) > 0$, $\phi_i = 1$ and $\frac{S_{ii}^2}{\lambda + S_{ii}^2} < 1$. Moreover, $\frac{\partial \lambda}{\partial q} \frac{\partial y^*}{\partial \lambda} \frac{\partial q}{\partial \hat{\p}}$ is of the form $a B B^T$ where $a \in \mathbb{R}_{++}$, forming a PSD matrix. $\mathbf{H}_i$ is positive-definite, hence the RHS (\ref{eq:cbf-halfspace}) is always negative and $u_0 = -(\alpha + \beta) v$ is a feasible solution. 
\end{proof}
Having access to a closed-form feasible point circumvents the costly routine of finding interior points, which is a required step when pruning redundant polytope constraints. However, when the robot is in collision, the corresponding constraints (\ref{eq:cbf-halfspace}) are not necessarily satisfied with $u_0$. We observe that there are typically very few ellipsoids in collision at any one time. Therefore, although we maintain all in-collision constraints, we are free to prune constraints for ellipsoids that are not in collision. Thus, collision and collision-free scenarios have approximately the same computation savings, which amounts to roughly pruning $99.9\%$ of all Gaussians.

\begin{corollary}
\textcolor{black}{
The quadratic program (\ref{opt:cbf_qp}) can be extended to enforce velocity and acceleration limits.}
\end{corollary}
\begin{proof}
\textcolor{black}{
A velocity barrier ($h_v(x) = v_{max}^2 - \lVert v \rVert^2 \geq 0$) constraint evaluated at $u_0$ for any $\kappa \in \R_{++}$ is 
\begin{equation}
\label{eq:vel_barrier}
    \underbrace{-2 v^T u_0}_{\dot{h}(x, u_0)} \geq -\kappa \cdot  h_v(x) \leftrightarrow \left(\kappa - 2(\alpha + \beta)\right) \lVert v \rVert^2 \leq \kappa v_{max}^2. 
\end{equation}
Note that the RHS of (\ref{eq:vel_barrier}) is always satisfied when $(\kappa, \alpha, \beta) > \mb{0}$ and $h_v(x) \geq 0$, hence the velocity barrier is forward invariant.
Consequently, given the same feasible solution and a desired limit on the acceleration $\lVert u \rVert^2 \leq a_{max}^2$, the following inequality must be satisfied
\begin{equation}
    \label{eq:control_limits}
    \lVert u_0 \rVert^2 = (\alpha + \beta)^2 \lVert v \rVert^2 \leq a_{max}^2  \; \forall \; \lVert v \rVert \leq v_{max},
\end{equation} 
hence $\alpha + \beta \leq \frac{a_{max}}{v_{max}}$.
Therefore, the inclusion of velocity and acceleration limits into (\ref{opt:cbf_qp}) does not compromise the feasibility or safety of the program since $u_0$ satisfies (\ref{eq:cbf-halfspace}, \ref{eq:vel_barrier}, and \ref{eq:control_limits}) so long as $\alpha, \beta$ are chosen accordingly. }
\end{proof}

\ifbool{condensed_paper}
{
}
{
If the state is unsafe, we can informally characterize unsafe regions into categories of steerable and unsteerable. Steerable regions correspond to positions where any control actions exist in the CBF action polytope, but the polytope no longer contains the origin. On the other hand, unsteerable regions are positions where two or more half-space intersections yield the empty set. Steerable regions occur as the robot transitions from free space to occupied. As an example, the robot begins to penetrate the surface Gaussians of a wall, and so a cluster of half-spaces begin closing in toward and past the origin. If these half-spaces close in too far, the empty set arises from the intersection. This scenario would correspond to when the robot has penetrated too far and finds itself stuck between Gaussians in the wall and those on the surface.

\textcolor{red}{I think we can make this subsection more concise.}

As the scene changes, and the Gaussian positions and extents vary with time, the CBF solver may find itself in either steerable or unsteerable regimes regardless of the state the robot was in before. If the robot finds itself in a steerable region, then the solver will be able to find a control action to move out into free space. However, if the robot is unsteerable (e.g. Gaussians have enveloped the robot), this behavior is typically an artifact of the scene representation and not necessarily indicative of a collision with the ground-truth geometry. More specifically, these artifacts could be ``floaters'' or simply where the scene is poorly supervised. Fortunately, these artifacts do subside as sensor data streams in and the GSplat continues to train. To discourage these artifacts from hurting controller performance a priori, the buffer distance $\epsilon$ can be increased to reflect the maximum distance Gaussians may change between consecutive solves of the CBF.
}

\section{Experiments}
\label{sec:results}
\begin{figure}
    \centering
    \includegraphics[trim={1cm 5.5cm 0cm 0cm}, width=\columnwidth]{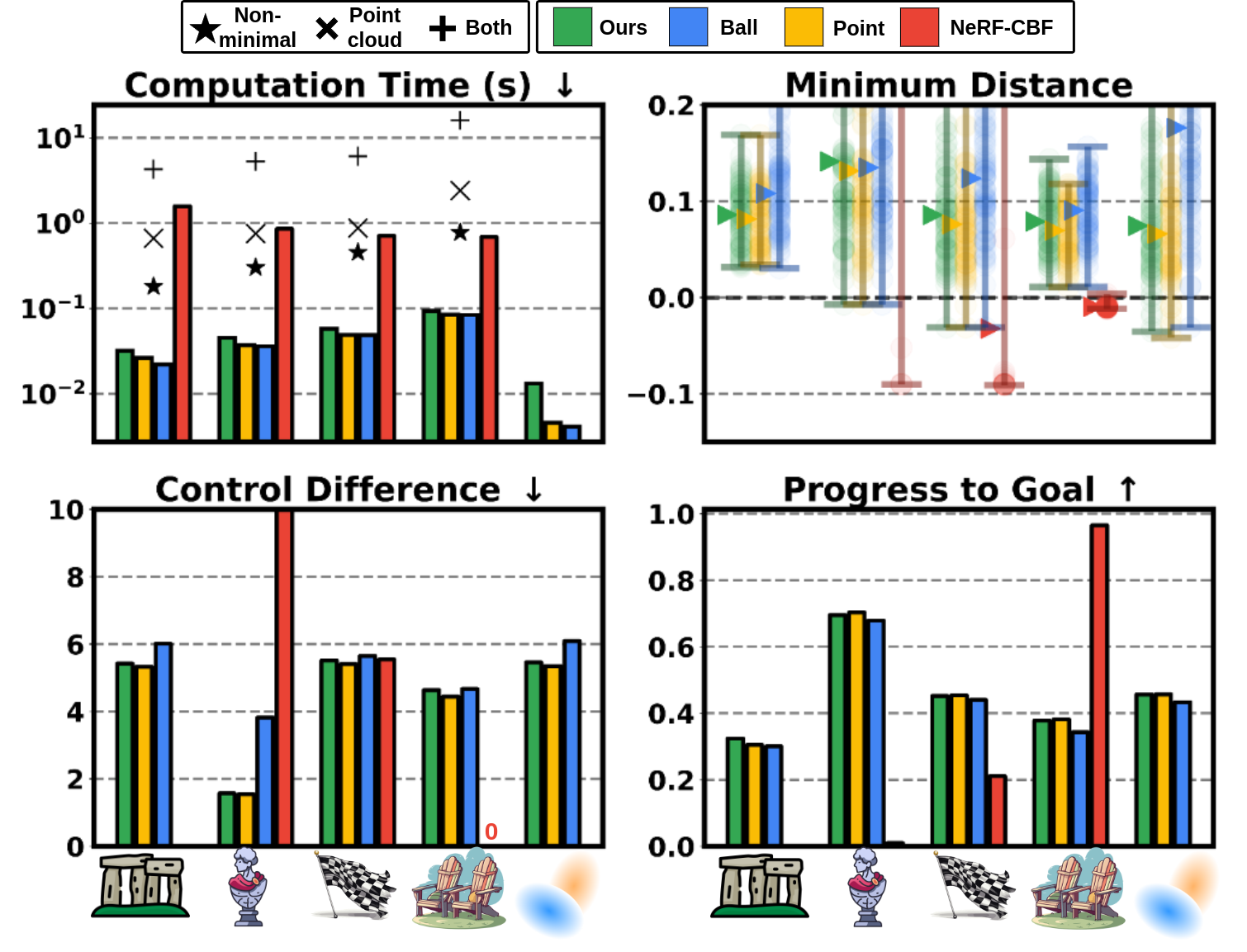}
    \caption{Computation time (lower is better), minimum distance to collision (above zero is safe), magnitude of control effort correction (lower is better), and progress to goal (higher is better) for five different scenes indicated by the icons. SAFER-Splat (ours, green), is compared with our CBF on a conservative bounding sphere of each ellipsoid (blue), our CBF on the GSplat means (yellow), and NeRF-CBF \cite{tong2023enforcing} (red). 
    Each trajectory is visualized as a translucent dot in the minimum distance and the wedge indicates the mean.
    }
    \vspace{-1em}
    \label{fig:sim-results}
\end{figure} 

\noindent \textbf{Simulation Results.}
We compare SAFER-Splat (green) to NeRF-CBF \cite{tong2023enforcing} (red), in addition two variants of our method: one which uses a conservative bounding sphere of each ellipsoid (blue) and one that uses the GSplat means (yellow) in \cref{fig:sim-results}. A Mahalanobis distance variant was also tested but proved too ill-behaved to show. NeRF-CBF is trained using Nerfacto \cite{nerfstudio} on the same dataset as the GSplat. We examine each method in five scenes: \emph{Stonehenge} (100K Gaussians), \emph{Statues} (200K), and \emph{Flightgate} (300K), \emph{Adirondacks} (500K), and \emph{low-res Flightgate} (15K).
All controllers are executed for $100$ trajectories. Our method strikes a favorable trade-off compared to other methods.

The computation time of SAFER-Splat and its variants are similar and more than an order of magnitude faster than NeRF-CBF. We also plot timing results without constraint pruning ($\star$), on a point cloud (about 6$\times$ more points than ellipsoids) represented by $\times$, and a combination of both ($+$), demonstrating that our real-time performance is attributed to both constraint pruning and the sparse GSplat representation. 

To assess safety, we compute the squared distance of the trajectories to the high-res GSplat \eqref{eq:our-cbf} (instead of COLMAP due to the better dense reconstruction from the GSplat). NeRF-CBF attains major violations (data points less than 0) compared to our method and its variants. Our method and its variants can become unsafe due to dynamics noise, but infrequently. 
We expect that in sparse reconstructions with large, elongated Gaussians, e.g., in the early stages of GSplat mapping, our method strikes a promising tradeoff between the conservativeness of the ball variant 
and the high potential for collision of the means-only variant. 
We see this behavior beginning to appear in \emph{low-res Flightgate} (15K Gaussians). 

We utilize the objective of \ref{opt:cbf_qp} (Control Difference) and progress to the goal ($ \min_t ||\p(t) - \p_0|| / ||\p_f - \p_0||$) as proxies for conservativeness. We do not expect the progress to be 1 because all methods are reactive and utilize a simple PD controller ($\bar{u}$) to navigate to the goal. We observe that NeRF-CBF is sporadic in its performance, failing completely on \emph{Stonehenge} and yielding only a few valid trajectories in \emph{Statues} and \emph{Flightgate}. Again, our  method and the point variant achieve similar results, which are slightly less conservative than the ball variant. 

\begin{figure}[]
    \centering
    \includegraphics[trim={0cm 2cm 0cm 0cm}, width=.99\columnwidth]{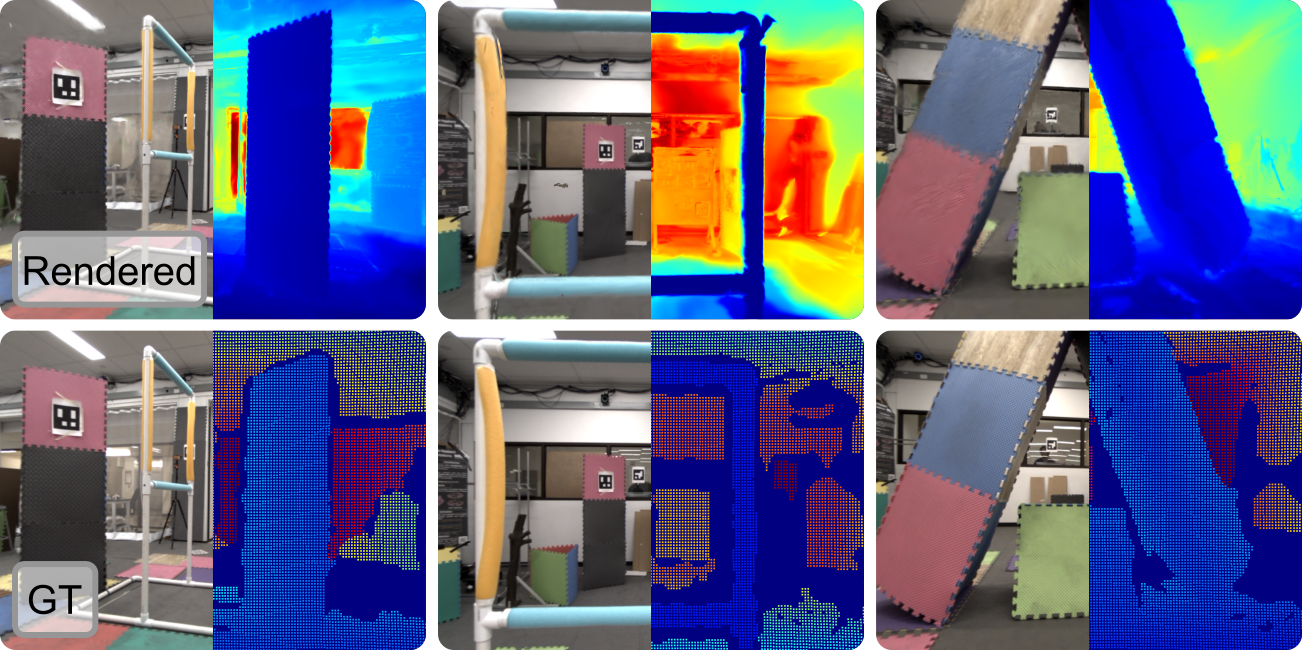}
    \caption{Top row: Overlaid color and depth renders from the GSplats produced using SAFER-Splat in the real-world experiments. Bottom row: Overlaid corresponding keyframe color image and sparse point cloud.}
    \label{fig:splat-results}
    \vspace{-1em}
\end{figure}

\begin{figure*}[h]
    \centering
    \includegraphics[trim={0cm 2cm 0cm 0cm}, width=.99\textwidth]{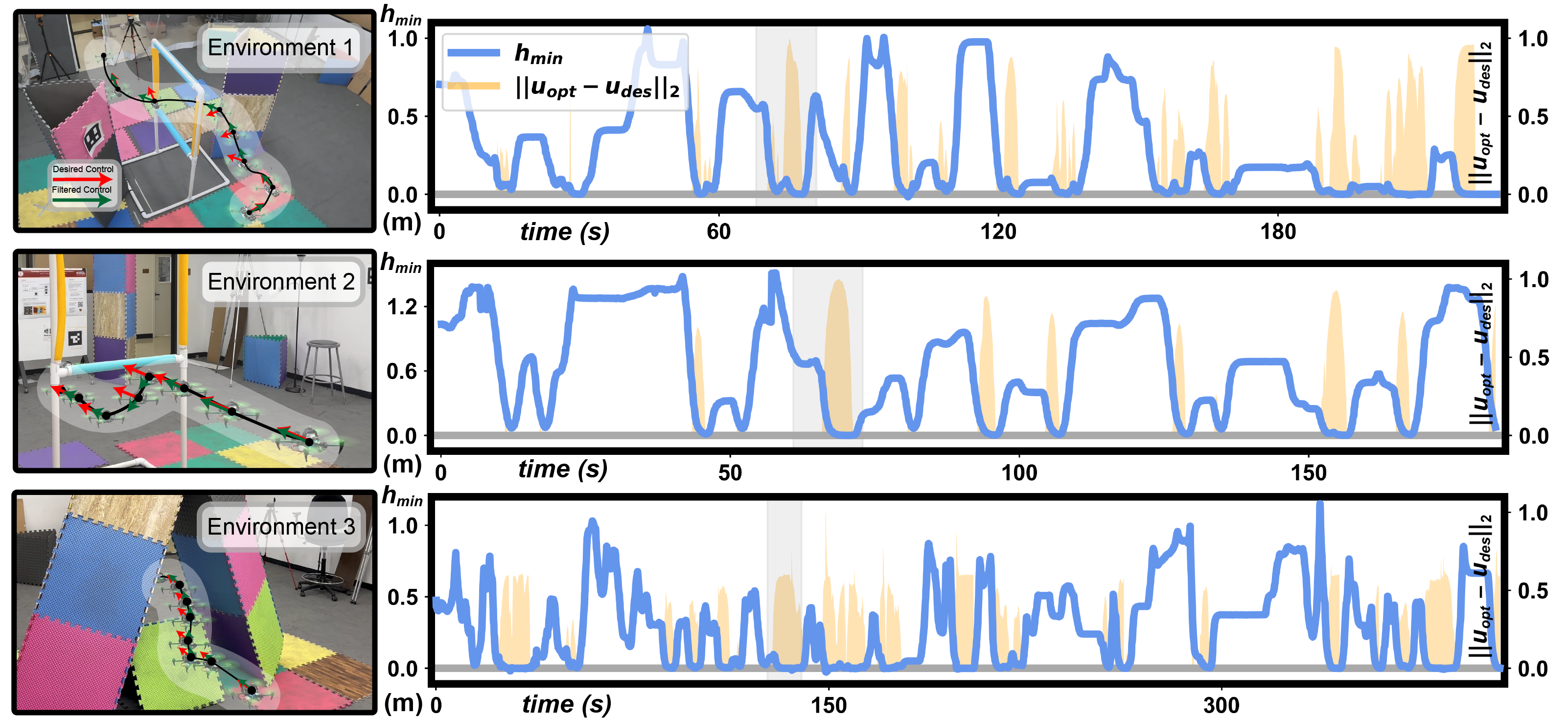}
    \caption{We showcase SAFER-Splat across three real-world scenes. A snapshot of the drone trajectory is shown on the left panel for each of the three scenes, annotated with the approximate desired and filtered control directions. In each of these trajectories, the operator attempted to directly hit the obstacle. On the right, we show the full flight trajectory, showing robustness with repeated collision avoidance. The pictured portion of each trajectory is highlighted in grey.}
    \label{fig:exp-results}
    \vspace{-1em}
\end{figure*}

\smallskip
\noindent \textbf{Hardware Platform.} Our experiments are executed on a 220 mm Starling 2 drone. The platform maintains a pose estimate using visual inertial odometry and tracks position, velocity, and acceleration commands with a low-level PX4 controller. Images, point clouds, and VIO poses are sent through ROS2 \cite{macenski2022robot} to an off-board desktop computer with RTX 4090 GPU via 5 GHz WiFi. The high-level drone acceleration commands come from a teleoperator and are relayed to the off-board computer which transmits CBF corrected commands back to the drone's low-level controller. SAFER-Splat's ROS nodes are made available in our codebase. Note that the scenes used in the real-world experiments contain visual fiducials and motion capture cameras, but these are \emph{not} used by our VIO or GSplat mapping modules for pose estimation.





\smallskip
\noindent \textbf{SplatBridge.} The incoming stream of RGB images, time-of-flight point clouds, and estimated poses are used to train a GSplat in real-time. We utilize an extension of NerfBridge \cite{yu2023nerfbridge}, called SplatBridge, to interface with the GSplat implementation in Nerfstudio \cite{nerfstudio}. At every keyframe, SplatBridge loads the image into a running buffer and seeds Gaussian primitives from the ToF point-cloud. Asynchronous to the sensor stream, SplatBridge optimizes the Gaussian attributes and keyframe camera poses, initialized at the VIO estimates, minimizing the loss proposed by \cite{kerbl3Dgaussians}. The SAFER-Splat safety layer updates its reference to the Gaussian primitives from SplatBridge before each query to the CBF.


\smallskip
\noindent \textbf{Hardware Experiments.}
In three real-world experiments, SAFER-Splat successfully enabled safe teleoperation even when the operator repeatedly attempted to force a collision by driving the drone towards obstacles. Fig. \ref{fig:exp-results} shows our sample trajectories from each scene along with the full plots of our safety value versus time, demonstrating that the safety layer avoids collision ($h_{min} > 0$) and only turns on close to the boundary (i.e., minimally invasive). Our method performed online optimization on the parameters of between 150,000 and 300,000 Gaussians in each scene while simultaneously computing safety filtered actions at 15Hz. We constructed our GSplat maps entirely while the drone was flying, and enabled our CBF after a brief warm-up period (5-30 seconds of flight). During flights in Environments 1 and 3, there were brief violations of the CBF, $h_{min} < 0$, in all cases these violations were on the order of 1-4mm and due to minor failures of the drone's low-level controller in accurately tracking the safety layer's outputs. Likely these tracking failures were caused by rotor induced aerodynamic effects when close to obstacles, and in no cases did they result in loss of control of the drone. Full recordings of each of these real-world flights are available on our project's website.
Fig. \ref{fig:splat-results} showcases the high visual accuracy of each environment's GSplat when trained via SplatBridge.

\section{Conclusion, Limitations, and Future Work}
\label{sec:conclusion}
We propose a real-time 3D reconstruction and safety filter pipeline using GSplats, which are photo-realistic, efficient to optimize, and interpretable. The proposed safety layer utilizes a CBF formulated for ellipsoidal robot geometries, and leverages the favorable properties of the GSplat representation for scalable and rigorous safety guarantees. SAFER-Splat is generalizable to arbitrary high-level planners, minimally-invasive, and well-behaved, resulting in smooth robot motion. The proposed safety layer is coupled with an online GSplat mapping module to enable safe navigation in unknown environments. 
In simulations, we show that SAFER-Splat out-performs comparable methods based on NeRFs, and in real-world experiments, we demonstrate that SAFER-Splat can simultaneously optimize a 3D reconstruction and perform safety filtering for a teleoperated drone.

SAFER-Splat assumes that the underlying GSplat map is relatively accurate with respect to the real-world, which does not hold in some situations, especially with dynamic objects.
Future work will seek to improve SAFER-Splat in these cases by accounting for dynamic objects in an uncertainty-aware optimization framework. Second, we hope to expand SAFER-Splat beyond double-integrator systems, e.g., contact-rich dynamical systems. Next, SAFER-Splat can be deployed on lower-end hardware as GSplat maps become more lightweight \cite{shorinwa2024fast}. Furthermore, SplatBridge can be made more robust to inaccurate camera pose estimates, which affect the accuracy of the GSplat. Finally, we seek to extend SAFER-Splat to semantic multi-robot mapping \cite{yu2025hammer, shorinwa2025siren}, semantic-aware safety, and simultaneous deployment with high-level planners like Vision-Language-Action models. 


\bibliographystyle{IEEEtran.bst}
\bibliography{citations.bib}

\end{document}